\pdfoutput=1
\documentclass[12pt,reqno]{amsart}

\numberwithin{equation}{section}
\usepackage[tt=false]{libertine}
\usepackage{mathtools}
\usepackage{amsmath}
\usepackage{amsfonts}
\usepackage{enumitem}

\usepackage{amssymb,mathrsfs}
\usepackage[varbb]{newpxmath}

\let\savedbigtimes\bigtimes
\let\bigtimes\relax
\usepackage{mathabx} 
\let\bigtimes\savedbigtimes

\usepackage[margin=1in]{geometry}
\usepackage{graphicx}
\usepackage{bbm}
\usepackage{hyperref,color}
\usepackage[capitalize,nameinlink]{cleveref}
\usepackage[dvipsnames]{xcolor}
\hypersetup{
	colorlinks=true,
	pdfpagemode=UseNone,
    citecolor=OliveGreen,
    linkcolor=NavyBlue,
    urlcolor=black,
	pdfstartview=FitW
}
\usepackage{appendix}
\crefname{appsec}{Appendix}{Appendices}
\usepackage{tikz}
\usepackage{bm}
\usepackage{mathtools}

\newtheorem{theorem}{Theorem}[section]

\newtheorem{lemma}[theorem]{Lemma}
\newtheorem{corollary}[theorem]{Corollary}
\newtheorem{conjecture}[theorem]{Conjecture}

\theoremstyle{definition}
\newtheorem{definition}[theorem]{Definition}

\newtheorem*{assumption*}{Assumption}

\crefname{lemma}{Lemma}{Lemmas}
\crefname{theorem}{Theorem}{Theorems}
\crefname{definition}{Definition}{Definitions}
\crefname{fact}{Fact}{Facts}
\crefname{claim}{Claim}{Claims}
\crefname{proposition}{Proposition}{Propositions}

\newcommand{\E}{\mathbb{E}}

\newcommand{\poly}{\mathrm{poly}}

\newcommand{\sgn}{\mathrm{sgn}}

\newcommand{\eps}{\varepsilon}
\renewcommand{\epsilon}{\varepsilon}

\newcommand{\N}{\mathbb{N}}

\newcommand{\R}{\mathbb{R}}

\renewcommand{\P}{\mathbb{P}}

\newcommand{\EE}{\mathbb{E}}

\newcommand{\beq}{\begin{equation}}
\newcommand{\eeq}{\end{equation}}

\renewcommand\Pr{\mathbf{P}}

\usepackage{appendix}
\crefname{appsec}{Appendix}{Appendices}

\usepackage{xspace}

{\bfseries}{\normalfont}

{\bfseries}{\itshape}

\DeclareMathOperator{\TV}{\mathrm{TV}}
\DeclareMathOperator{\KL}{\mathrm{KL}}

\newcommand{\Z}{\mathbb{Z}}
\newcommand{\iid}{\text{i.i.d.}}
\newcommand{\CLWE}{\text{CLWE}}
\newcommand{\nn}{\textsf{nn}}
\newcommand{\NN}{\textsf{NN}}

\newcommand{\RR}{\mathbb{R}}

\newcommand{\sA}{\mathcal{A}}

\newcommand{\sF}{\mathcal{F}}

\newcommand{\clwe}{\mathrm{CLWE}}
\newcommand{\lwe}{\mathrm{LWE}}

\newcommand{\inner}[1]{\left\langle #1 \right\rangle}

\usepackage[
backend=bibtex,
style=alphabetic,
citestyle=alphabetic,
giveninits=false,
maxnames=99,
date=year,
backref=true
]{biblatex}
\AtEveryBibitem{%
\ifentrytype{article}{
    \clearfield{url}%
    \clearfield{urldate}%
    \clearfield{eprint}%
    \clearfield{review}%
    \clearfield{location}%
    \clearfield{address}%
    \clearfield{issn}%
    \clearfield{isbn}%
    \clearfield{doi}%
    \clearfield{series}
    }{}%
\ifentrytype{inproceedings}{
    \clearlist{publisher}%
    \clearlist{location}%
    \clearfield{url}%
    \clearfield{urldate}%
    \clearfield{issn}%
    \clearfield{isbn}%
    \clearfield{doi}%
    \clearfield{series}
}{}%
\ifentrytype{book}{
    \clearfield{url}%
    \clearfield{urldate}%
    \clearfield{review}%
    \clearfield{location}%
    \clearfield{address}%
    \clearfield{issn}%
    \clearfield{isbn}%
    \clearfield{doi}%
    \clearfield{series}
}{}%
\ifentrytype{inbook}{
    \clearfield{url}%
    \clearfield{urldate}%
    \clearfield{review}%
    \clearfield{location}%
    \clearfield{address}%
    \clearfield{issn}%
    \clearfield{isbn}%
    \clearfield{doi}%
    \clearfield{eprint}%
    \clearfield{series}
}{}%
\ifentrytype{incollection}{
    \clearfield{url}%
    \clearfield{urldate}%
    \clearfield{review}%
    \clearfield{location}%
    \clearfield{address}%
    \clearfield{issn}%
    \clearfield{isbn}%
    \clearfield{doi}%
    \clearfield{series}
}{}
}

\DeclareFieldFormat[inproceedings]{title}{#1} 
\DeclareFieldFormat[article]{title}{#1}
\DeclareFieldFormat[inbook]{title}{#1}
\DeclareFieldFormat[misc]{title}{#1}
\DeclareFieldFormat[article]{year}{#1}
\DeclareFieldFormat[book]{title}{#1}

\renewbibmacro{in:}{}
\renewbibmacro*{volume+number+eid}{%
  \printfield{volume}
  \setunit{\addcolon\space}%
  \printfield{number}%
  \printfield{eid}}

\title[On the Hardness of Learning One Hidden Layer Neural Networks]{On the Hardness of Learning One Hidden Layer Neural Networks}

\author[S. Li, I. Zadik, M. Zampetakis]{Shuchen Li$^\mathsection$, Ilias Zadik$^\mathsection$, Manolis Zampetakis$^\mathsection$}
\thanks{
$^\mathsection$ Yale University. \\ Emails:  \texttt{shuchen.li@yale.edu}, \texttt{ilias.zadik@yale.edu}, \texttt{manolis.zampetakis@yale.edu}}

\begin{document}

\begin{abstract}
In this work, we consider the problem of learning one hidden layer ReLU neural networks with inputs from $\R^d$. We show that this learning problem is hard under standard cryptographic assumptions even when: (1) the size of the neural network is polynomial in $d$, (2) its input distribution is a standard Gaussian, and (3) the noise is Gaussian and polynomially small in $d$. Our hardness result is based on the hardness of the Continuous Learning with Errors (CLWE) problem, and in particular, is based on the largely believed worst-case hardness of approximately solving the shortest vector problem up to a multiplicative polynomial factor. 
\end{abstract}
\maketitle

\section{Introduction} \label{sec:intro}

In this paper, we examine the fundamental computational limitations of learning neural networks in a distribution-specific setting. Our focus is on the following canonical regression scenario: let \( f \) be an unknown target function that can be represented as a simple neural network, let \( \mathcal{D} \) be a \( d \)-dimensional distribution from which samples \( x_i \) are drawn, i.e., \( x_i \sim \mathcal{D} \), and let \( \eta_i \sim {N}(0, \sigma^2) \) be small Gaussian observation noise. The statistician receives \( m \) independent and identically distributed samples of the form \( (x_i, f(x_i) + \eta_i) \) for \( i = 1, \ldots, m \) with the goal of constructing an estimator \( \hat{f} \) that is computable in polynomial time and achieves a small mean squared error (MSE) on a new sample drawn from \( \mathcal{D} \). Specifically, we aim to minimize
\(
\mathbb{E}_{x \sim \mathcal{D}} \left[ \lvert \hat{f}(x) - f(x) \rvert^2 \right].
\)
We consider the following two objectives in terms of the MSE:
\begin{enumerate}
  \item \textit{Achieving Vanishing MSE}: Obtaining an MSE that approaches zero as the dimension \( d \) increases.
  \item \textit{Weak Learning}: Attaining an MSE slightly better than that of the trivial mean estimator. Formally, this means ensuring for large enough $d$:
\[
\mathbb{E}_{x \sim \mathcal{D}} \left[ \lvert \hat{f}(x) - f(x) \rvert^2 \right] \leq \operatorname{Var}_{x \sim \mathcal{D}}(f) - \frac{1}{\operatorname{poly}(d)}.
\]
\end{enumerate}

It is well known due to \cite{klivans2009cryptographic} that without further assumptions on the distribution $\mathcal{D}$, e.g., when $\mathcal{D}$ can be supported over the Boolean hypercube, learning even one-hidden layer neural networks is impossible (or ``hard''\footnote{Following a standard convention, we refer to a computational task as ``hard'' if it is impossibile for polynomial-time methods.}) for polynomial-time estimators under standard cryptographic assumptions. Given the success of neural networks in practice, a long line of recent work has attempted to study instead the canonical continuous input distribution case where $\mathcal{D}$ is the isotropic Gaussian, i.e., $\mathcal{D}={N}\left(0,I_d\right)$ which is also the setting that we follow in this work. Yet, despite a long line of research, the following important question remains open.
\begin{center}
  \emph{Is there a poly-time algorithm for learning 1 hidden layer neural networks when $\mathcal{D} = {N}\left(0,I_d\right)$?}
\end{center}

It is known that a single neuron, i.e., 0-hidden layer neural network, can be learned in polynomial time \cite{zarifisrobustly}, while neural networks with more that 2 hidden layers are hard to learn \cite{chen2022hardness}. Nevertheless, the case of 1-hidden layer neural networks is not well understood. In this paper we close this gap in the literature by answering the question above. We show that it is hard to learn 1-hidden layer neural networks under Gaussian input assuming the hardness of some standard cryptographic assumptions. Our result settles an important gap in the computational complexity of learning neural networks with simple input distributions $\mathcal{D}$ as we explain in Section~\ref{sec:contributions} below.
\subsection{Prior work}
We now provide more details on the literature prior to this work. We first remind the reader that, formally, polynomial-sized 1-hidden layer neural networks can be expressed using some width parameter $k=\poly(d),$ some weights $w_i \in \mathbb{R}^d$ and some $a_i, b_i \in \mathbb{R}$ as follows
\[f(x)=\sum_{i=1}^{k} a_i (\langle x_i,w_i\rangle+b_i)_+.\]  Now for this class of single hidden layer neural networks, a powerful algorithmic toolbox has been created under the Gaussian input assumption including the works of \cite{janzamin2015beating, brutzkus2017sgd, ge2017learning,zhong2017recovery,allen2018learning, zhang2019learning,bakshi2019learning, diakonikolas2020algorithms, awasthi2021efficient, SZB21-cosine-learning}. 
Interestingly most of these proposed algorithmic constructions assume the so-called ``realizable'' (or noiseless) case where $\sigma=0$. Yet, with the important exception of the brittle lattice-based methods used in \cite{SZB21-cosine-learning}, the techniques used are customarily expected to be at leastly mildly robust to noise, and in particular generalize to the most realistic case where $\sigma$ is positive but polynomially small. Another yet significantly more concerning restriction of the above positive results is that they all require some \emph{assumptions on the weights}. For example, a common such required assumption is that the weights $w_i, i=1,2,\ldots,m$ are linearly independent (see e.g., \cite{awasthi2021efficient} and references therein). It is natural to wonder whether requiring any such assumption is necessary for any polynomial-time estimator to learn the class of one hidden layer neural networks, or simply an artifact of the employed techniques.

In that direction, researchers have managed to establish certain unconditional and conditional lower bounds for this problem. Specifically, in terms of conditional lower bounds, \cite{goel2020superpolynomial} and \cite{diakonikolas2020algorithms} proved that under a worst-case choice of weights the class of the so-called correlation Statistical Query (cSQ) methods (containing e.g., gradient descent with respect to the squared loss) fails to learn the class of one hidden layer neural networks with super-constant width even in the noiseless regime. With respect to unconditional lower bounds, \cite{SZB21-cosine-learning} has proven that under cryptographic assumptions (specifically the continuous Learning with Errors (CLWE) assumption), if the noise per sample $\eta_i$ is allowed to be polynomially small but \emph{adversarially chosen} (and not Gaussian) then no polynomial-time $\hat{f}$ can succeed\footnote{Formally, the lower bound in \cite{SZB21-cosine-learning} is about the cosine activation function (and not the ReLU we assume in this work). Yet, standard approximation results \cite{Bach16-approx} can transfer the result to one hidden layer neural networks at the cost of extra polynomially small additive approximation error (see also \cite[Appendix E]{SZB21-cosine-learning})}. Although both are quite interesting results, they unfortunately come with their drawbacks. The cSQ model is known in many settings to be underperforming compared to multiple other natural polynomial-time methods (see e.g., \cite{andoni2019attribute,chen2022learning,damian2024computational}), which arguably limits the generality of such an unconditional lower bound. Moreover, while \cite{SZB21-cosine-learning} is now a lower bound against all polynomial-time algorithms one could argue that the computational hardness arises exactly because of the addition of \emph{adversarial} noise and may not be inherent to learning one hidden layer neural networks. In particular, as we mentioned in our main question above, it remains an intriguing open problem in the above literature whether some polynomial-time estimator can in fact learn the whole class of polynomial-size one hidden layer neural networks under small Gaussian noise.

We note that a cleaner hardness picture has been established when the neural networks have at least two hidden layers. \cite{daniely2021local} has proven, via an elegant lifting technique, that cryptographic assumptions (specifically the existence of a local pseudorandom generator with polynomial stretch) imply that learning three hidden layers neural networks is computationally hard even in the noiseless case where $\sigma=0$. Moreover, \cite{chen2022hardness} built on the lifting technique of \cite{daniely2021local} and proved that under different cryptographic assumptions (specifically the learning with rounding assumption) that learning two hidden layers neural networks is also computationally hard again in the noiseless case. On top of that, \cite{chen2022hardness} also proved a general Statistical Query (SQ) lower bound in this case of two hidden layer neural networks. Albeit these powerful recent results, it remains elusive whether a similar technique can prove the hardness for the more basic case of one hidden layer neural networks, something also highlighted as one of the main open questions in \cite{chen2022hardness}.

\subsection{Contribution} \label{sec:contributions}

In this work, we establish that under the CLWE assumption from cryptography \cite{bruna2020continuous} learning the class of one hidden layer neural network with polynomial small Gaussian noise is indeed computationally hard. Importantly, solving CLWE in polynomial-time implies a polynomial-time quantum algorithm that approximates within polynomial factors the \emph{worst-case} Gap shortest vector problem (GapSVP), a widely believed hard task in cryptography and algorithmic theory of lattices \cite{micciancio2009lattice}. Interestingly, our lower bound holds even under the requirement of \emph{weakly learning} the neural network. We present our findings in the following informal theorem.

\begin{theorem}[Informal; see Theorem~\ref{thm:mainTheorem}]
    Let $\mathcal{F}_k$ the class of width $k$ one hidden layer neural networks and arbitrary noise variance $\sigma=1/\mathrm{poly}(d).$ For any $k=\omega(\sqrt{d \log d}),$ if there exists a polynomial-time algorithm that can weakly learn $\mathcal{F}_k$ under Gaussian noise of variance $\sigma$ then there exists a polynomial-time quantum algorithm that approximates GapSVP within a $\poly(d)$ factor. 
\end{theorem}

The above result settles the computational question of learning one hidden layer neural networks under polynomially small Gaussian noise. It is perhaps natural to wonder if we can also obtain a lower bound against even smaller levels of noise. 

First, we highlight that as we also mentioned in the Introduction, this is already a significantly small amount of noise; most natural algorithmic schemes in learning theory are mildly robust to noise, and therefore they can tolerate polynomially-small levels of Gaussian noise (if not a constant level). That being said, we also mention that one can prove a more general version of our result by combining the reductions between CLWE and classical LWE \cite{gupte2022continuous}; if a polynomial-time estimator can weakly learn $\mathcal{F}_k$ for some $\omega(\sqrt{d \log d \cdot \log(d/\sigma)})=k=\poly(d)$ under Gaussian noise of \emph{arbitrary} variance $\sigma=\sigma_d$ such that \(\log(1/\sigma) = \poly(d)\), then there exists also a polynomial-time quantum algorithm that approximates GapSVP within a factor $\poly(1/\sigma,d)$.  In particular, given that the current state-of-the-art algorithm for GapSVP remains since 1982 the celebrated Lenstra-Lenstra-Lov\'{a}sz (LLL) lattice basis reduction algorithm \cite{lenstra1982factoring} which has approximation factor $\exp(\Theta(d))$, we prove that any learning algorithm for one hidden layer neural networks succeeding for any $\exp(-o(\sqrt{d}))\leq \sigma$ would immediately imply a major breakthrough in the algorithmic theory of lattices (see Section~\ref{sec:superpolynomial} for a lengthier discussion on this and more details on this connection). 

The only case that is left open by our results is that some very brittle algorithm can learn in polynomial-time the class of one hidden layer neural networks (only) for exponentially small values of $\sigma$. In fact, that is proven to be the case using the brittle LLL-based methods for the case of cosine neuron in \cite{SZB21-cosine-learning}, and for multiple other ``noiseless'' settings in the recent learning theory literature \cite{andoni2017correspondence,zadik2018high,gamarnik2021inference,zadik2022lattice,diakonikolas2022non}. Yet, while we believe this is an interesting and potentially highly non-trivial theoretical question, the value of any such brittle algorithmic method in learning or statistics is unfortunately unclear since a non-negligible amount of noise always exists in these cases.

\subsection{Organization}

We begin in Section~\ref{sec:definitions-and-notations} with the the formulation of PAC-learning neural networks and the necessary preliminaries on lattice-based cryptography that we utilize to present our hardness result. Then in Section~\ref{sec:mainResult} we state formally our main result and we provide a proof sketch. In Sections~\ref{sec:clwe_lip} and \ref{sec:clwe_1hlnn} we provide the proof of our result in two steps. First, we show the hardness of learning any single periodic neural network, and then we show how this implies the hardness of learning 1-hidden layer neural networks.

\section{Preliminaries}
\label{sec:definitions-and-notations}

\subsection{Notations}
Throughout the paper we use the standard asymptotic notation, $o, \omega, O,\Theta,\Omega$ for comparing the growth of two positive sequences $(a_d)_{d \in \N }$ and $(b_d)_{d \in \N }$: we say $a_d=\Theta(b_d)$
if there is an absolute constant $c>0$ such that $1/c\le a_d/ b_d \le c$; $a_d =\Omega(b_d)$ or $b_d = O(a_d)$ if there exists  an absolute constant $c>0$ such that $a_d/b_d \ge c$; and $a_d =\omega(b_d)$ or $b_d = o(a_d)$ if $\lim _d a_d/b_d =0$. We say $x=\poly(d)$ if for some $r>0$ it holds $x=O(d^r)$. Let \(N(\mu, \sigma^2)\) denote the Gaussian distribution with mean \(\mu\) and variance \(\sigma^2\), and \(N(\mu, \Sigma)\) denote the multivariate Gaussian distribution with mean \(\mu\) and covariance \(\Sigma\).

\subsection{PAC-learning with Gaussian input distribution.} Our focus on this work is the problem of learning a sequence of real-valued function classes $\{\sF_d\}_{d \in \N}$, each over the standard Gaussian input distribution on $\RR^d$. The input is a multiset of i.i.d.~labeled examples $(x,y) \in \RR^d \times \RR$, where $x \sim N(0,I_d)$, $y = f(x) + \xi$, $f \in \sF_d$, and $\xi \sim {N}(0,\sigma^2)$ for some $\sigma^2>0.$ We denote by $D=D_f$ the resulting data distribution. The goal of the learner is to output an hypothesis $h: \mathbb{R}^d \rightarrow \mathbb{R}$ that is close to the target function $f$ in the squared loss sense over the Gaussian input distribution.


Throughout the paper we define $\ell: \RR \times \RR \rightarrow \RR_{\ge 0}$ the squared loss function defined by $\ell(y,z) = (y-z)^2$. For a given hypothesis $h$ and a data distribution $D$ on pairs $(x,z) \in \mathbb{R}^d \times \mathbb{R}$, we define its \emph{population loss} $L_D(h)$ over a data distribution $D$ by
\begin{align}\label{pop_loss_gen}
    L_D(h) = \EE_{(x,y) \sim D} [\ell(h(x),y)] \;. 
\end{align}

We now define the important notion of weak learning.
\begin{definition}[Weak learning]
\label{def:weak-learning}
Let $\eps=\eps(d) > 0$ be a sequence of numbers, $\delta \in (0,1)$ a fixed constant, and let $\{\sF_d\}_{d \in \N}$ be a sequence of function classes defined on input space $\RR^d$. We say that a (randomized) learning algorithm $\sA$ $\eps$-weakly learns $\{\sF_d\}_{d \in \N}$ over the standard Gaussian distribution if for every $f \in \sF_d$ the algorithm outputs a hypothesis $h_d$  such that for large values of $d$ with probability at least $1-\delta$
\begin{align*}
    L_{D_f}(h_d) \le L_{D_f}(\EE[f(x)])-\eps\;.
\end{align*}
Note that $\EE_{x \sim {N}(0,I_d)}[f(x)]$ is the best predictor agnostic to the input data in this setting. Hence, we refer to $L_D(\EE[f(x)])=\mathrm{Var}_{Z \sim N(0,I_d)}(f(Z)),$ as the trivial loss, and $\eps$ as the edge of the learning algorithm.
\end{definition}For simplicity, we refer to an hypothesis as \textit{weakly learning} a function class if it can achieve edge $\epsilon$ which is depending inverse polynomially in $d$. Moreover, we simply set from now on $\delta=1/3$ when we refer to weak learning.

\subsection{Worst-Case Lattice Problems}
\label{app:lattice-problems}
Some background on lattice problems is required for our work. We start with the definition of a lattice.
\begin{definition}
Given linearly independent $b_1,\ldots, b_d \in \mathbb{R}^d$, let 
\begin{align} \Lambda=\Lambda(b_1,\ldots,b_d)=\left\{\sum_{i=1}^{d} \lambda_i b_i : \lambda_i \in \mathbb{Z}, i=1,\ldots,d \right\}~, 
\end{align} 
which we refer to as the lattice generated by $b_1,\ldots,b_d$.
\end{definition}

A core worst-case \emph{decision} algorithmic problem on lattices is GapSVP. In GapSVP, we are given an instance of the form $(\Lambda,t)$, where $\Lambda$ is a $d$-dimensional lattice and $t \in \RR$, the goal is to distinguish between the case where $\lambda_1(\Lambda)$, the $\ell_2$-norm of the shortest non-zero vector in $\Lambda$, satisfies $\lambda_1(\Lambda) < t$ from the case where $\lambda_1(\Lambda) \ge \alpha(d) \cdot t$ for some ``gap'' $\alpha(d) \ge 1$. We refer to any such successful algorithm as solving GapSVP within an $\alpha(d)$ factor.

GapSVP is known to be NP-hard for ``almost'' polynomial approximation factors, that is, $2^{(\log d)^{1-\eps}}$ for any constant $\eps > 0$, assuming problems in $\mathrm{NP}$ cannot be solved in quasi-polynomial time~\cite{khot2005hardness,haviv2007tensor}. Moreover, importantly for this work, GapSVP is strongly believed to be computationally hard (even with quantum computation), for \emph{any} polynomial approximation factor $\alpha(d)$ \cite{micciancio2009lattice}, as described in the following conjecture. 
\begin{conjecture}[{\cite[Conjecture 1.2]{micciancio2009lattice}}]
\label{conj:svp-poly-factor}
There is no polynomial-time quantum algorithm that solves $\mathrm{GapSVP}$ to within polynomial factors.
\end{conjecture}

\noindent We comment on the version of this conjecture for super-polynomial factors in Section~\ref{sec:superpolynomial}.

\subsection{Continuous Learning with Errors (CLWE)~\cite{bruna2020continuous}.}
Of crucial importance to us is the CLWE decision (or detection) problem. We define the CLWE distribution $\clwe_{\beta,\gamma}$ on dimension $d$ with frequency $\gamma=\gamma(d) \ge 0$, and noise rate $\beta = \beta(d) \ge 0$ to be the distribution of i.i.d. samples of the form $(x,z)\in \mathbb{R}^d \times [-1/2,1/2)$ where $x \sim N(0,I_d) , \xi \sim N(0,\beta)$, $w$ uniformly chosen from the sphere $\mathcal{S}^{d-1}$ and 
\begin{align}
z = \gamma \inner{x,w}+\xi \mod 1~. \label{CLWE}
\end{align}
Note that for the $\mathrm{mod}$ 1 operation, we take the representatives in $[-1/2,1/2)$.
The CLWE problem consists of detecting between i.i.d. samples from the CLWE distribution or the null distribution $N(0,I_d) \times U([-1/2,1/2))$ which we denote by $A_0$.

Given $\gamma=\gamma(d)$ and $\beta=\beta(d)$, we consider a sequence of decision problems $\{\clwe_{\beta,\gamma}\}_{d\in\N}$, indexed by the input dimension $d$, in which the learner receives $m$ samples from an unknown distribution $D$ such that either $D =A_{\beta,\gamma}$ or $D=A_0$. We consider the classical hypothesis testing setting that we aim to construct a polynomial-time binary-valued testing algorithm $\sA$ which uses as input the samples and distinguishes the two distributions. Specifically, $\sA$ takes values in $\{A_{\beta,\gamma}, A_0\}$ and seeks to output ``$A_{\beta,\gamma}$'' when $D=A_{\beta,\gamma}$ and ``$A_0$'' when $D=A_0$. Under this setup, we define the \emph{advantage} of $\sA$ to be the following difference, 
\begin{align*}
    \Bigl| \Pr_{x \sim (A_{\beta,\gamma})^{\otimes m}}[\sA(x) = A_0] - \Pr_{x \sim A_0^{\otimes m}}[\sA(x) = A_0] \Bigr|
    \; .
\end{align*}
Note that the advantage simply equals to one minus the sum of the type I and type II errors in statistical terminology. We call the advantage \emph{non-negligible} if it decays at most polynomially fast i.e., it is $\Omega(d^{-C})$ for some $C>0.$ 

\cite{bruna2020continuous} provided worst-case evidence based on the hardness of GapSVP (Conjecture~\ref{conj:svp-poly-factor}) that solving the CLWE decision problem with non-negligible advantage is computationally hard for any $\beta=1/\poly(d)$ as long as $\gamma \ge 2\sqrt{d}$. This is an immediate corollary of the result below combined with Conjecture~\ref{conj:svp-poly-factor}.

\begin{theorem}[{\cite[Corollary 3.2]{bruna2020continuous}}]
\label{thm:clwe-hardness}
Let $\beta = \beta(d) =1/\poly(d)$ and $2\sqrt{d} \leq \gamma =\gamma(d)= \poly(d)$. Then, if there exists a polynomial-time algorithm for $\clwe_{\beta,\gamma}$ with non-negligible advantage, then there is a polynomial-time quantum algorithm for solves $\mathrm{GapSVP}$ within $\poly(d)$ factors.
\end{theorem}

For simplicity, we say that some algorithm ``solves $\clwe$'' to refer to the fact that the algorithm has non-negligible advantage for the decision version of $\clwe$.

\section{Main Result} \label{sec:mainResult}

We begin with defining the class $\mathcal{F}_k^\NN$ of one hidden layer neural networks 
\[\mathcal{F}_k^\NN = \left\{ f_{W,b}(x) = \sum_{j=1}^k a_j(\langle w_j, x \rangle + b_j)_+ \mid a\in \R^k,  W\in \R^{d\times k}, b\in \R^{k}  \right\}. \label{eq:1nn}\] 

Our main result is the following.
\begin{theorem} \label{thm:mainTheorem_0}
Let \(d\in \N\), and arbitrary \(\sigma = \poly(d)^{-1}\), \(\varepsilon = \poly(d)^{-1}\), and \(k = \omega(\sqrt{d \log d})\). Then a polynomial-time estimator that \(\varepsilon \)-weakly learns the function class $\mathcal{F}_k^\NN$ over Gaussian inputs \(x \overset{\iid}{\sim} N(0,I_d)\) under Gaussian noise \(\xi\overset{\iid}{\sim} N(0,\sigma^2)\) implies a polynomial-time quantum algorithm that approximates $\mathrm{GapSVP}$ to within polynomial factors. 
\end{theorem}Notice that directly from our Theorem~\ref{thm:mainTheorem_0} and the widely believed Conjecture~\ref{conj:svp-poly-factor} we can conclude that no polynomial-time estimator can weakly learn the class of one hidden layer neural networks under arbitrary polynomially small Gaussian noise.

\subsection{Proof Sketch and Comparison with \cite{SZB21-cosine-learning}}
Our (simple) proof is an appropriate combination of two key steps. We first establish in Section~\ref{sec:clwe_lip} that solving the CLWE problem reduces to learning Lipschitz periodic neurons under polynomially small \emph{Gaussian} noise (see Theorem~\ref{thm:CLWE-to-phi}). This is a direct improvement upon the key result \cite[Theorem 3.3]{SZB21-cosine-learning} that establishes that CLWE reduces to learning Lipschitz periodic neurons under polynomially small \emph{adversarial} noise. Our approach is to perhaps interestingly show that one can ``Gaussianize'' the adversarial noise in the labels generated via the reduction followed by \cite{SZB21-cosine-learning} by simply injecting additional Gaussian noise of appropriate variance to them (see Lemma~\ref{lem:noisy-tv-close}).

Recall that, using standard approximation results \cite[Appendix E]{SZB21-cosine-learning}, learning 1-Lipschitz periodic neurons under polynomially small adversarial noise is equivalent with learning one hidden layer neural networks of appropriate polynomial width under (slightly larger) polynomially small adversarial noise. Unfortunately, we cannot straightforwardly generalize this logic to Gaussian errors using our first step, because the induced approximation error can in principle be too large for our ``Gaussianization'' lemma to work. Regardless, instead of using approximation results, in Section~\ref{sec:clwe_1hlnn}, we follow a more direct route and prove that for any arbitrary large bounded interval $[-R,R]$ one can explicitly construct an appropriate Lipschitz periodic neuron and a polynomial-width neural network that exactly agree on $[-R,R]$, i.e., have zero ``approximation'' error (see Lemma~\ref{lem:phi-nn-eq}). This lemma combined with our first step Theorem~\ref{thm:CLWE-to-phi} allow us to reduce CLWE to learning one hidden layer neural networks under Gaussian noise. Combining the above with the reduction from GapSVP to CLWE (Theorem~\ref{thm:clwe-hardness}) let us then conclude Theorem~\ref{thm:mainTheorem_0}.

\section{CLWE reduction to Lipschitz Periodic Neurons under Gaussian noise}\label{sec:clwe_lip}



We first recall the notion of Lipschitz periodic neurons from \cite{SZB21-cosine-learning}. Let $\gamma = \gamma(d) > 1$ be a sequence of numbers indexed by the input dimension $d \in \NN$, and let $\phi: \RR \rightarrow [-1,1]$ be a \(1\)-Lipschitz and 1-periodic function. We denote by $\sF_\gamma^\phi$ the function class
\begin{align}
\label{periodic-neurons}
    \sF_\gamma^\phi = \{f: \RR^d \rightarrow [-1,1] \mid f(x) = \phi(\gamma \langle w, x \rangle), w \in S^{d-1}\}
\end{align}
Note that each member of the function class $\sF_\gamma^\phi$ is fully characterized by a unit vector $w \in S^{d-1}$. We refer such function classes as \emph{Lipschitz periodic neurons}.

\cite{SZB21-cosine-learning} has established that solving CLWE reduces to learning in polynomial-time the class $\sF_\gamma^\phi$ under polynomially small adversarial noise. Their reduction is very simple; given a CLWE $(x,\gamma \langle w, x \rangle+\xi \mod 1)$  one can create a sample $(x,\phi(\gamma \langle w, x \rangle+\xi)$ by applying $\phi$ since $\phi$ is 1-periodic. But since $\phi$ is $1$-Lipschitz and $\xi$ is ``small'', notice $\phi(\gamma \langle w, x \rangle+\xi) = \phi(\gamma \langle w, x \rangle)+\xi'$ for some $\xi'$ also ``small'' as $|\xi'| \leq |\xi|.$ Hence, the authors of \cite{SZB21-cosine-learning} construct from a CLWE sample, a sample from the Lipschitz periodic neuron class $\sF_\gamma^\phi$, but under the somewhat cumbersome noise variable $\xi'$ which we can only control its magnitude -- for this reason $\xi'$ is referred to as small adversarial noise in \cite{SZB21-cosine-learning}.

Our first step is to improve upon \cite{SZB21-cosine-learning} and construct instead a sample from the Lipschitz periodic neuron class $\sF_\gamma^\phi$, but under simply Gaussian noise $\xi'$. Our idea to do so is to simply inject additional small Gaussian noise to $\phi(\gamma \langle w, x \rangle+\xi).$ We prove that as long as the variance of the added noise is of slightly larger magnitude than the magnitude of the (already polynomially small) noise $\xi$, in total variation distance the sample approximately equals in distribution to $\phi(\gamma \langle w, x \rangle)+\xi'$ where now $\xi'$ is Gaussian.  This result is described in the following lemma.
\begin{lemma}
\label{lem:noisy-tv-close}
Let \(\phi:\R\to \R\) be an \(1\)-Lipschitz function. For fixed \(\gamma>0\) and \(w\in S^{d-1}\), and \(x\sim N(0, I_d)\), \(\xi_0 \sim N(0,\beta)\), \(\xi \sim N(0, \sigma^2)\), the total variation distance between the distributions of \((x, \phi(\gamma \langle w, x \rangle+\xi_0)+\xi)\) and \((x, \phi(\gamma \langle w, x \rangle)+\xi)\) is at most \(\frac{\sqrt{\beta}}{\sqrt{2\pi}\sigma}\).
\end{lemma}

\begin{proof}
Since the first entries of the two pairs are the same, it suffices to upper bound the total variance distance between the distributions of \(z_1=\phi(\gamma \langle w, x \rangle+\xi_0)+\xi\) and \(z_2=\phi(\gamma \langle w, x \rangle)+\xi\) conditioning on \(x\). Note that conditioning on \(\xi_0\) and $x$, the distribution of \(z_1\) is \(N(\phi(\gamma \langle w, x \rangle+\xi_0), \sigma^2)\) and the distribution of \(z_2\) is \(N(\phi(\gamma \langle w, x \rangle), \sigma^2)\). Thus,
\begin{align*}
&\TV(z_1|x,z_2|x) \le \E_{\xi_0\sim N(0,\beta)}[\TV(z_1|(\xi_0, x),z_2|(\xi_0, x))]
\le \E_{\xi_0\sim N(0,\beta)} \left[ \sqrt{\frac{\KL(z_1|(\xi_0, x) \,\|\, z_2|(\xi_0, x))}{2}} \right] \\
& \qquad  = \E_{\xi_0\sim N(0,\beta)} \left[ \frac{|\phi(\gamma \langle w, x \rangle + \xi_0) - \phi(\gamma \langle w, x \rangle)|}{2\sigma} \right] 
\le \E_{\xi_0\sim N(0,\beta)} \left[ \frac{|\xi_0|}{2\sigma} \right] =  \frac{\sqrt{\beta}}{\sqrt{2\pi}\sigma},
\end{align*} 
where the first inequality is from the triangle inequality, the second inequality is from Pinsker's inequality, the equality in the third line is from the KL divergence between two single dimensional Guassians, and the last inequality is from the \(1\)-Lipschitz continuity of \(\phi\). 
\end{proof}

 Lemma~\ref{lem:noisy-tv-close} allow us to establish the following key CLWE hardness result for Lipschitz periodic neurons.

\begin{theorem}
\label{thm:CLWE-to-phi}
Let \(d\in \N\) , \(\gamma = \omega(\sqrt{\log d})\), \(\sigma \in (0,1)\), \(\varepsilon = \poly(d)^{-1}\), \(m_1 = \poly(d)\). Moreover, let \(\phi:\R\to [-1,1]\) be an \(1\)-Lipschitz, \(1\)-periodic function. Then, a polynomial-time learning algorithm using \(m_1\) samples that \(\varepsilon \)-weakly learns the function class $\mathcal{F}_{\gamma}^\phi$ over Gaussian inputs \(x \overset{\iid}{\sim} N(0,I_d)\) and under Gaussian label noise \(\xi\overset{\iid}{\sim} N(0,\sigma^2)\) implies a polynomial-time algorithm for \(\CLWE_{\beta,\gamma}\) for any \(\beta \le \min \left\{ \frac{\sigma^2}{10^4 m_1^2}, \frac{\varepsilon ^2}{10^3} \right\}\).
\end{theorem}

We defer the proof of the theorem to Section~\ref{sec:pf_CLWE}. Notice that a direct corollary of Theorem~\ref{thm:CLWE-to-phi} is that a weak learning algorithm for the class $\mathcal{F}_{\gamma}^{\phi}$, implies a quantum algorithm for approximating GapSVP withing polynomial factors.

\begin{corollary}
\label{cor:phi-solves-SVP}
Let \(d \in \N\), \(\gamma = \poly(d)\) with \(\gamma\ge 2\sqrt{d}\), \(\sigma = \poly(d)^{-1}\), \(\varepsilon = \poly(d)^{-1}\). Moreover, let \(\phi:\R\to [-1,1]\) be an \(1\)-Lipschitz \(1\)-periodic function. Then, a polynomial-time algorithm that \(\varepsilon\)-weakly learns the function class \(\mathcal{F}_{\gamma}^\phi = \{f_{\gamma,w}(x) = \phi(\gamma \langle w, x \rangle) \mid w\in S^{d-1}\}\) over Gaussian inputs \(x \overset{\iid}{\sim} N(0,I_d)\) under Gaussian noise \(\xi\overset{\iid}{\sim} N(0,\sigma^2)\) implies a polynomial-time quantum algorithm that approximates $\mathrm{GapSVP}$ to within polynomial factors. 
\end{corollary}

\begin{proof}
Since the weak learning algorithm runs in polynomial time, the number of samples it uses is \(m_1 = \poly(d)\). Let \(\beta = \min \left\{ \frac{\sigma^2}{10^4 m_1^2}, \frac{\varepsilon ^2}{10^3} \right\} = \poly(d)^{-1}\). By Theorem~\ref{thm:CLWE-to-phi}, there is a polynomial-time algorithm for \(\CLWE_{\beta,\gamma}\). Moreover, since \(\beta = \poly(d)^{-1}\) and \(2\sqrt{d} \le \gamma = \poly(d)\), from Theorem~\ref{thm:clwe-hardness}, there is a polynomial-time quantum algorithm that solves \(\mathrm{GapSVP}\) within \(\poly(d)\) factors.
\end{proof}



\subsection{Proof of Theorem~\ref{thm:CLWE-to-phi}}\label{sec:pf_CLWE}
\begin{proof}
We begin with introducing some definitions and notation about $\clwe$ and about weak learners for the class $\mathcal{F}_{\gamma}^{\phi}$ that will be useful to us during the proof.

\begin{description}[leftmargin=12pt]
    \item[CLWE] We denote by \(P_0\) the CLWE distribution, i.e., samples \((x_i,z_i = \gamma \left\langle w,x_i \right\rangle + \xi_{0,i}\bmod{1})\) where \(w\sim U(S^{d-1})\), \(x_i \sim N(0,I_d)\), and \(\xi_{0,i}\sim N(0,\beta)\). Let \(Q_0\) denote the null distribution, i.e., samples \((x_i,y_i)\) where \(x_i \sim N(0,I_d)\), and \(y_i\sim U[0,1]\). 
    For some appropriately large $m_2=\poly(d)$ that will be determined in the proof, we study whether a polynomial-time algorithm for \(\CLWE_{\beta,\gamma}\) can distinguish between \(m\) i.i.d.\ samples from \(P_0\) and \(m\) i.i.d.\ samples from \(Q_0\), for \(m = m_1+m_2=\poly(d)\), with non-negligible advantage. 
    \smallskip

    For a labeled sample \((x,y) \in \R^d\times \R\), we define \begin{align}F_\xi(x,y) = (x,\phi(y)+\xi). \label{eq:F_xi}
    \end{align}Let \(P_1\) and \(Q_1\) denote the distributions of \(P_0\) and \(Q_0\) after applying \(F_{\xi}\) for \(\xi \sim N(0,\sigma^2)\). 
    \medskip

    \item[Weak Learning $\mathcal{F}^{\phi}_\gamma$] Let \(P_\phi\) denote the distribution of the samples \((x,\phi(\gamma \langle w, x \rangle)+\xi)\). For \(\varepsilon = \poly(d)^{-1}\) a weak learner for $P_{\phi}$ is a polynomial time learning algorithm \(\mathcal{A}\) that take as input \(m_1\) samples from \(P_{\phi}\) and with probability \(2/3\) outputs a hypothesis \(h':\R\to\R\) such that \(L_{P_\phi}(h')\le L_{P_\phi}(\E[f_{\gamma,w}(x)])- \varepsilon \). Since we are using the squared loss, \(\tilde h(x) = \sgn(h'(x))\cdot \min(|h'(x)|,1)\) is always no worse than \(h'(x)\), since \(\phi(x)\in [-1,1]\) and the noise on the label is unbiased. Thus, we can assume without loss of generality that \(h'(x) \in [-1,1]\).
\end{description}

Our goal in this proof is to assume access to a weak learner for the function class $\mathcal{F}_{\gamma}^{\phi}$ with $m_1$ samples, and design an algorithm for solving the $\clwe$ problem (as defined above) with $m=m_1+m_2$ for an appropriate number of additional samples $m_2=\poly(d)$. More precisely, we want to design an efficient algorithm $\mathcal{B}$ to distinguish between a set of samples from $P_0$ and a set of samples from $Q_0$ with $m=m_1+m_2, m_2=\poly(d)$ samples, using a weak learning algorithm $\mathcal{A}$ for $\mathcal{F}_{\gamma}^{\phi}$ with $m_1$ samples.
\medskip

\noindent \textit{Definition of Algorithm $\mathcal{B}$.}
For a sufficiently large $m_2=\poly(d)$ for the purposes of the proof the follows, we are given \(m=m_1+m_2=\poly(d)\) i.i.d.\ samples $\{(x_i, z_i)\}_{i = 1}^{m}$ from an unknown distribution \(D\), which is either \(P_0\) or \(Q_0\), algorithm $\mathcal{B}$ follows the following steps.
\begin{enumerate}
  \item Sample \(\xi_i \overset{\iid}{\sim} N(0,\sigma^2)\), \(i=1,2,\dots, m\).
  
  \item For each \(i=1,2,\dots, m\), apply \(F_{\xi_i}\), defined in \eqref{eq:F_xi}, to $(x_i, z_i)$ to get a sample $(x_i, s_i)$ from \(D_1\), which is either \(P_1\) or \(Q_1\).
  
  \item Run \(\mathcal A\) on the first \(m_1\) of the samples from $D_1$, and let \(h:\R\to [-1,1]\) be the hypothesis that $\mathcal{A}$ outputs.
  
  \item Generate $m_2$ samples from $Q_1$.
  
  \item Compute the empirical loss $\hat L_{D_1}(h)$ of \(h\) on the remaining \(m_2\) samples from $D_1$, and the empirical loss $L_{Q_1}(h)$ on the \(m_2\) samples generated from \(Q_1\).
  \item Test whether \(\hat L_{D_1}(h) \le \hat L_{Q_1}(h) - \varepsilon /5\) or not. \label{algstep6}
  \item In the end, conclude \(D=P_0\) if \(h\) passes the test in step (6) and \(D=Q_0\) otherwise.
\end{enumerate}

\noindent \textit{Proof of Correctness of $\mathcal{B}$.} Next we prove the correctness of this algorithm $\mathcal{B}$ assuming the correctness of $\mathcal{A}$ and using Lemma~\ref{lem:noisy-tv-close}. We first show that if $D=P_0$ then $h$ will pass the test \(\hat L_{D_1}(h) \le \hat L_{Q_1}(h) - \varepsilon /5\) and then we show that if $D = Q_0$ then $h$ will fail this test.
\medskip

\noindent \textbf{Case I: \(\mathbf{D=P_0}\).} Recall that \(h\) and \(h'\) denote the output of \(\mathcal A\) given \(m_1\) samples from \(D_1=P_1\) and \(P_\phi\) respectively, employing the notation we introduced above. By the data processing inequality, $TV(h,h') \le \TV\bigl(P_1^{\otimes m_1},P_\phi^{\otimes m_1}\bigr) \le m_1\cdot \TV(P_1,P_\phi),$ where $\TV(h, h')$ refers to the total variation between the distribution of $h$, and the distribution of $h'$. Hence, by Lemma~\ref{lem:noisy-tv-close}, $\TV(h, h')$ is upper bounded by \(\frac{m_1\sqrt{\beta}}{\sqrt{2\pi}\sigma}\). Then, 
\begin{align*}
    &\left| \P_{h\gets \mathcal{A}(P_1^{\otimes m_1})}[L_{P_\phi}(h) \le L_{P_\phi}(\E[f_{\gamma,w}(x)]) - \varepsilon ] - \P_{h'\gets \mathcal{A}(P_\phi^{\otimes m_1})}[L_{P_\phi}(h) \le L_{P_\phi}(\E[f_{\gamma,w}(x)]) - \varepsilon ]\right| \\
   & \qquad \qquad \le 2\TV(h,h') \le \frac{m_1 \sqrt{2\beta}}{\sqrt{\pi}\sigma} \le 0.01,
\end{align*}
since we have chosen \(\beta \le \frac{\sigma^2}{10^4 m_1^2}\). 
Thus, we have with probability at least \(2/3 - 0.01\) that 
\begin{align}
L_{P_\phi}(h)\le L_{P_\phi}(\E[f_{\gamma,w}(x)])- \varepsilon.\label{eq:phiProofPhilossH}
\end{align}

Note that \(L_{P_\phi}(h) = \E_{(x,z)\sim P_\phi}(h(x)-z)^2 = \E_{P_\phi}(h(x)-f_{\gamma,w}(x))^2 + \E \xi^2\), and similarly, \(L_{P_1}(h) = \E_{P_1}(h(x)-\phi(\gamma \langle w, x \rangle + \xi_0))^2 + \E\xi^2\). Hence, 
\begin{align}
    \left| L_{P_\phi}(h)-L_{P_1}(h) \right| &= \left|\E_{P_\phi}(h(x)-f_{\gamma,w}(x))^2 - \E_{P_1}(h(x)-\phi(\gamma \langle w, x \rangle + \xi_0))^2 \right| \notag \\
    &= \left| \E_{x,\xi_0} \bigl(\phi(\gamma \langle w, x \rangle) - \phi(\gamma \langle w, x \rangle+\xi_0)\bigr)\bigl(\phi(\gamma \langle w, x \rangle) + \phi(\gamma \langle w, x \rangle+\xi_0)-2h(x)\bigr) \right| \notag \\
    &\le 4\E_{x,\xi_0}|\phi(\gamma \langle w, x \rangle) - \phi(\gamma \langle w, x \rangle+\xi_0) | \le 4\sqrt{2\beta/\pi} \le \varepsilon/5, \label{eq:phiProofPphiP1}
\end{align}
where the last inequality is because we have chosen \(\beta \le \frac{\varepsilon ^2}{10^3 }\). 
Let \(c=\E_{y\sim U[0,1]}[\phi(y)]\) then, canceling \(\E\xi^2\) similarly, 
\begin{align}
    \left| L_{P_\phi}(c) - L_{Q_1}(c) \right| &= \left| \E_{P_\phi}(c-\phi(\gamma \langle w, x \rangle))^2 - \E_{Q_1}(c-\phi(y))^2 \right| \notag \\
    &= \left| \E_{y\sim P_y}(c-\phi(y))^2 - \E_{y\sim U[0,1]}(c-\phi(y))^2 \right| \notag \\
    &\le 2 \left\| (c-\phi(y))^2 \right\|_\infty\cdot \TV(P_y,U[0,1]) \notag \\
    &\le 16 \exp(-2\pi^2 \gamma^2) \le o(\poly(d)^{-1}) \le \varepsilon /5, \label{eq:phiProofPphiQ1}
\end{align}
where \(P_y\) denotes the distribution of \((\gamma \langle w, x \rangle \bmod 1)\) for \(x\sim N(0, I_d)\), the second inequality is from \cite[Claim I.6]{SZB21-cosine-learning}, and the last inequality is because \(\varepsilon = \poly(d)^{-1}\). Combining \eqref{eq:phiProofPhilossH}, \eqref{eq:phiProofPphiP1}, and \eqref{eq:phiProofPphiQ1} we get that with probability at least \(2/3-0.01\), 
\begin{align}
    L_{P_1}(h)& \le L_{P_\phi}(h) + \varepsilon /5 \le L_{P_\phi}(\E[f_{\gamma,w}(x)])-4 \varepsilon /5 \notag \\
    &\le  L_{P_\phi}(c)-4 \varepsilon /5 \le L_{Q_1}(c) - 3\varepsilon /5 \le L_{Q_1}(h) - 3\varepsilon /5, \label{eq:phiProofP1Q1}
\end{align}
where the third inequality is from the optimality of \(\E[f_{\gamma,w}(x)]\) among constant predictors for \(P_\phi\), and the last inequality is from the optimality of \(c\) among all predictors for \(Q_1\). 

Using the remaining \(m_2\) samples \((x_i,z_i)\) from \(P_0\), and the newly generated \(m_2\) samples \( (x_i',y_i')\) from \(Q_0\), \(i=m_1+1,\dots,m\), compute the empirical losses \(\hat L_{P_1}(h) = \frac{1}{m_2}\sum_{i=m_1+1}^{m} \ell\bigl(h(x_i),F_{\xi_i}(z_i)\bigr) \) and \(\hat L_{Q_1}(h) = \frac{1}{m_2}\sum_{i=m_1+1}^{m} \ell\bigl(h(x_i'),F_{\xi_i}(y_i')\bigr)\). We know that \(|h(x) - \phi(y)| \le 2\), and \(\xi\) is Gaussian with variance \(\sigma^2 < 1\). Hence \(h(x) - F_{\xi}(y)\) is sub-Gaussian with some absolute constant parameter. Therefore \(\ell(h(x),F_{\xi}(y))\) is sub-exponential with some absolute constant parameters. Then by Bernstein's inequality, \(| \hat L_{D_1}(h) - L_{P_1}(h) | \le \varepsilon /5\) and \(| \hat L_{Q_1}(h) - L_{Q_1}(h) | \le \varepsilon /5\), both with probability at least \(1-\exp(-\min\{\Omega(m_2\varepsilon ^2), \Omega(m_2 \varepsilon )\})\) which is \(1-o(1)\), as long as \(m_2 = \omega(\varepsilon ^{-2})\), which can be \(\poly(d)\). Using these concentration bounds together with \eqref{eq:phiProofP1Q1} we get that with probability at least \(2/3-0.01-o(1)\), it holds that 
\[\hat L_{D_1}(h) \le L_{P_1}(h) + \varepsilon /5 \le L_{Q_1}(h) - 2 \varepsilon /5 \le \hat L_{Q_1}(h) - \varepsilon /5 \]
hence $h$ will pass the test of the step~\ref{algstep6} of algorithm $\mathcal{B}$ and $\mathcal{B}$ will return the correct answer.
\medskip

\noindent \textbf{Case II: $\mathbf{D = Q_0}$.} In this case $D_1 = Q_1$. Using $m_2 = \omega(\varepsilon^{-2})$ large enough and applying Bernstein's inequality again we get that \(| \hat L_{D_1}(h) - L_{Q_1}(h) | \le \varepsilon /20\) and \(| \hat L_{Q_1}(h) - L_{Q_1}(h) | \le \varepsilon /20\), both with probability at least \(1-\exp(-\Omega(m_2\varepsilon ^2))=1-o(1)\). Which means that \(| \hat L_{D_1}(h) - \hat L_{Q_1}(h) | \le \varepsilon /10\) with probability at least \(1-o(1)\). Hence, \(\hat L_{D_1}(h) > \hat L_{Q_1}(h) - \varepsilon /5\) and the test in step~\ref{algstep6} of the algorithm $\mathcal{B}$ fails.
\medskip

In both cases, the test correctly concludes \(D=P_0\) or \(D=Q_0\), by using the empirical loss \(\hat L_{D_1} (h)\) and comparing it to the value \(\hat L_{Q_1}(h) - \varepsilon /5\). 
\end{proof}

\section{The Cryptographic Hardness of Learning One Hidden Layer Neural Networks} \label{sec:clwe_1hlnn}


In this section we complete the proof of Theorem~\ref{thm:mainTheorem_0}. To do this we construct a family of one hidden layer neural networks with polynomial size that is 1-Lipschitz and 1-periodic over a finite range $[-R, R]$. Because our input $x$ is Gaussian, and hence has ``light'' tails, we show that this is enough to apply our Theorem~\ref{thm:CLWE-to-phi} from the previous section and conclude our hardness result.
\medskip

We begin by reminding the reader the class $\mathcal{F}_k^\NN$  of one hidden layer neural networks defined in \eqref{eq:1nn}. Let us consider the following function, \[\phi(x) = \left| x-\frac{3}{4}- \left\lfloor x-\frac{1}{4} \right\rfloor \right|-\frac{1}{4}= \begin{cases}
    x - k, & x\in [k-1/4,k+1/4], \\
    1/2 - (x - k), & x\in [k+1/4, k+3/4],
\end{cases}\qquad k\in \Z,\] which is \(1\)-periodic, \(1\)-Lipschitz, and \(|\phi(x)|\le 1/4\) for all \(z\in \R\) (see Figure~\ref{fig:phi-nn}). The following lemma shows that it interestingly coincides with an one hidden layer neural network on some interval. 

\begin{lemma}
\label{lem:phi-nn-eq}
For \(R \in \N\), let \[\nn(x) = (x+R)_+ - (x-R)_+ + 2 \sum_{k=1}^{2R} \left( x+R+\frac{1}{4}-k \right)_+ - \left( x+R+\frac{3}{4}-k \right)_+. \]
Then, \(\nn(x) = \phi(x) \cdot \mathbf{1}\{x \in [-R,R]\}\). 
\end{lemma}
\begin{proof}
For \(x \le -R\), all the ReLU functions evaluate to \(0\), and \(\nn(x) = 0\). For \(x \ge R\), all the ReLU functions evaluate to \(\textsf{id}\), and \(
    \nn(x) = (x+R) - (x-R) + 2\sum_{k=1}^{2R}\left( x+R+\frac{1}{4}-k \right) - \left( x+R+\frac{3}{4}-k \right) = 2R - 2R = 0\). 
The interesting case is of course when \(x\in [-R, R]\). Observe that for \(k=1,2,\dots, 2R\), \[\left( x+R+\frac{1}{4}-k \right)_+ - \left( x+R+\frac{3}{4}-k \right)_+ = \begin{cases}
    0, & x < -R-3/4+k, \\
    -1/2, & x > -R-1/4+k, \\
    -x-R-3/4+k, & \text{otherwise}.
\end{cases}\]
If \(x \in [k-1/4,k+1/4]\) for some \(k\in \Z\), then \(\nn(x) = (x+R) + 2(-1/2)(R+k) = x - k = \phi(x)\). If \(x \in [k+1/4, k+3/4]\) for some \(k\in \Z\), then \(\nn(x) = (x+R) + 2(-x-R-3/4+R+k+1)+2(-1/2)(R+k)=1/2-(x-k) = \phi(x)\). 
\end{proof}

\begin{figure}[htbp]
\centering
\includegraphics[width=0.8\linewidth]{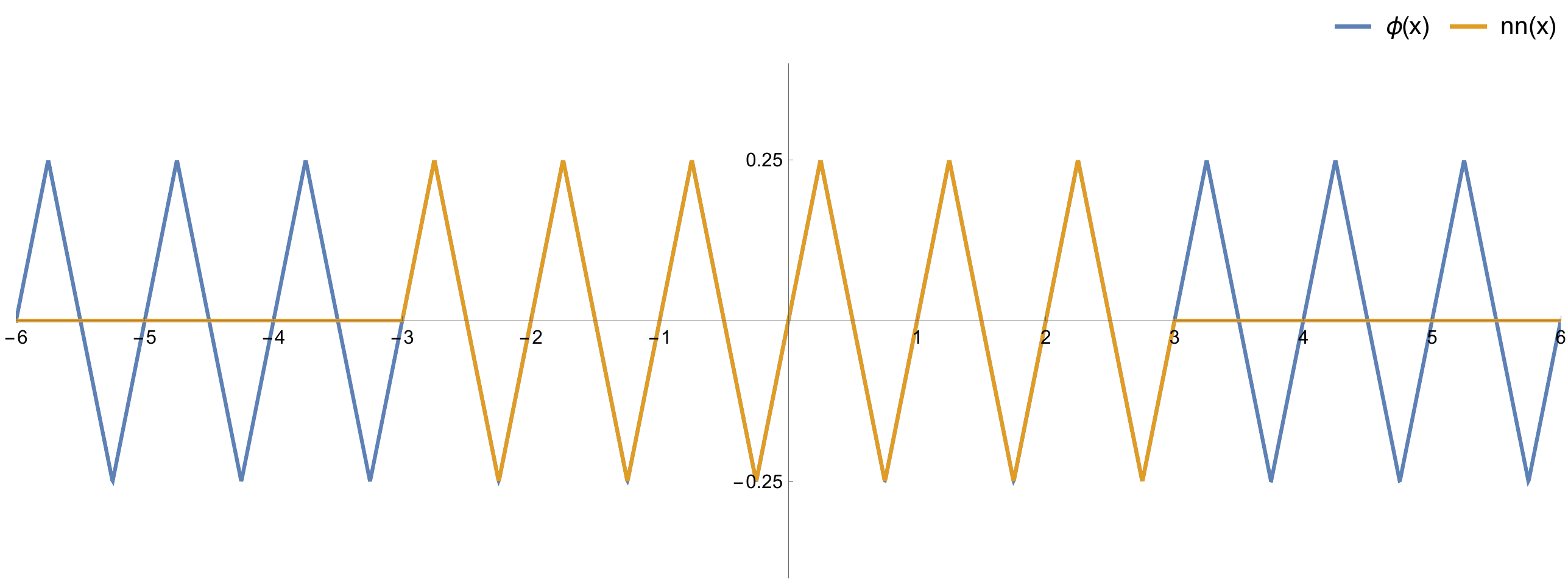}
    \caption{\(\phi(x)\) and \(\nn(x)\) for \(R=3\)}
    \label{fig:phi-nn}
\end{figure}
We next consider an instantiation of \(\nn\) from Lemma~\ref{lem:phi-nn-eq} that coincides with \(\phi\) on \([-R, R]\) for arbitrary \(R = \omega(\gamma \sqrt{\log d})\). Observe that $\nn$ which takes a single input, one hidden layer neural network with \((4R+2)\) ReLU neurons. We can then define the multivariate version of $\nn$ as \(\NN(x) = \nn(\gamma \langle w, x \rangle) \), which is also an one hidden layer neural network. This way we can define the following subclass of one hidden layer neural networks $\mathcal{F}_k^\NN$ given by,
\begin{align}\mathcal{F}_{R}^\NN = \left\{ f(x) = \nn(\gamma \langle w, x \rangle) \mid w \in \R^{d}, \gamma \in \R  \right\}, \label{eq:NN_class}\end{align}
which contains one hidden layer neural networks with width \(O(R)\). 
Recall from the previous section that, for \(x \sim N(0, I_d)\) and \(\xi \sim N(0, \sigma^2)\), \(P_\phi\) denotes the distribution of \((x, \phi(\gamma \langle w, x \rangle) + \xi)\), which is the input for learning the periodic function \(\phi\). Let \(P_{\NN}\) denote the distribution of \((x, \NN(x) + \xi)\), which is the input for learning the one hidden layer neural network \(\NN\). We next show that samples generated from $P_{\phi}$ are essentially the same as samples generated from $P_{\NN}$.
\begin{lemma}
\label{lem:phi-nn-tv-close}
For \(R\in \N\), the total variance distance between \(P_\phi\) and \(P_\NN\) is upper bounded by $O \left( \frac{\gamma \exp(-R^2/2\gamma^2)}{\sigma R}\right).$
When \(\gamma = \poly(d)\), and \(R = \gamma \sqrt{\omega(\log d) + 2\log(1/\sigma)}\), the total variation distance is 
\[O(\poly(d) \exp(- \omega(\log d))) = d^{-\omega(1)}.\]
\end{lemma}

\begin{proof}
Note that conditioning on \(x\), \(z_1=\phi(\gamma \langle w, x \rangle)+\xi\) and \(z_2=\nn(\gamma \langle w, x \rangle)+\xi\) are Gaussians with mean  \(\phi(\gamma \langle w, x \rangle)\) and \(\nn(\gamma \langle w, x \rangle)\), and variance \(\sigma^2\). Thus, \begin{align*}
    &\TV(P_\phi, P_\NN) \le \E_{x\sim N(0, I_d)} [\TV(z_1|x, z_2|x)] 
    \le \E_{x\sim N(0, I_d)} \left[ \frac{|\phi(\gamma \langle w, x \rangle) - \nn(\gamma \langle w, x \rangle)|}{2\sigma} \right] \\
    &\qquad \qquad = \frac{1}{2\sigma} \E_{x\sim N(0, 1)} [|\phi(\gamma x) - \nn(\gamma x) |] 
    \le \frac{1}{8\sigma} \Pr_{x\sim N(0, 1)} [|x| \ge R/\gamma] \\
    &\qquad \qquad \qquad \le \frac{1}{4\sigma} \frac{\exp(-R^2/2\gamma^2)}{\sqrt{2\pi }R/\gamma} 
    \le O \left( \frac{\gamma \exp(-R^2/2\gamma^2)}{\sigma R}\right). 
\end{align*}
\end{proof}

From Lemmas~\ref{lem:phi-nn-eq} and~\ref{lem:phi-nn-tv-close} we have that there exists a subclass $\mathcal{F}_R^{\NN}$ of one hidden layer neural networks that produces the same polynomially-many samples as the class $\mathcal{F}_\gamma^{\phi}$ for this carefully picked 1-Lipschitz and 1-periodic function $\phi$. We next show that in fact a learning algorithm for $\mathcal{F}_R^{\NN}$ implies a learning algorithm for $\mathcal{F}_\gamma^{\phi}$.

\begin{lemma}
\label{lem:phi-to-nn}
Let \(d\in \N\), \(\gamma = \poly(d)\), \(\sigma = \sigma(d) \in (0, 1)\), \(\varepsilon = \poly(d)^{-1}\). Moreover, let \(\phi(x) = \left| x-3/4 - \left\lfloor x-1/4 \right\rfloor \right|-1/4\), and \(R = \gamma \sqrt{\omega(\log d) + 2\log(1/\sigma)}\). Then a polynomial-time algorithm that \(\varepsilon \)-weakly learns the function class \(\mathcal{F}_R^\NN = \left\{ f_{w}(x) = \nn(\gamma \langle w, x \rangle) \mid w \in \R^{d}, \gamma \in \R  \right\}\) over Gaussian inputs \(x \overset{\iid}{\sim} N(0,I_d)\) under Gaussian noise \(\xi\overset{\iid}{\sim} N(0,\sigma^2)\) implies a polynomial-time algorithm that \(\frac{\varepsilon }{2}\)-weakly learns the function class \(\mathcal{F}_{\gamma}^\phi = \{f_{\gamma,w}(x) = \phi(\gamma \langle w, x \rangle) \mid w\in S^{d-1}\}\) over the same input and noise distribution. 
\end{lemma}
\begin{proof}
Let \(\nn\) be the function in Lemma~\ref{lem:phi-nn-eq} that coincides with \(\phi\) on \([-R, R]\) for arbitrary $R$ satisfying \(R = \gamma \sqrt{\omega(\log d) + 2\log(1/\sigma)}\). Then, let \(\NN(x) = \nn(\gamma \langle w, x \rangle)\), and thus \(\NN \in \mathcal{F}_R^\NN\). 

For \(\varepsilon = \poly(d)^{-1}\), and \(\mathcal{A}\) be a polynomial-time learning algorithm that takes as input \(m = \poly(d)\) samples from \(P_\NN\) and with probability \(2/3\) outputs a hypothesis \(h':\R\to \R\) such that \(L_{P_\NN}(h') \le L_{P_\NN}(\E[\NN(x)]) - \varepsilon \). Since we are using the squared loss, \(\tilde h(x) = \sgn(h'(x))\cdot \max(|h'(x)|,1/4)\) is always no worse than \(h'(x)\), as \(\NN(x)\in [-1/4,1/4]\) and the noise on the label is unbiased. Thus, we can assume without loss of generality that \(h'(x) \in [-1/4,1/4]\). 

To learn the function class \(\mathcal{F}_\gamma^\phi\) given \(m=\poly(d)\) samples from \(P_\phi\), run \(\mathcal{A}\) directly on these samples, which gives \(h\), and output \(h\). Similarly, by the data processing inequality, \[\TV(h, h') \le \TV(P_\phi^{\otimes m}, P_\NN^{\otimes m}) \le m\cdot \TV(P_\phi, P_\NN).\]

By Lemma~\ref{lem:phi-nn-tv-close}, this is upper bounded by \( m\cdot d^{-\omega(1)} = d^{-\omega(1)} < 0.01\). Thus, with probability at least \(2/3 - 0.01\), we have \(L_{P_\NN}(h) \le L_{P_\NN}(\E[\NN(x)]) - \varepsilon \). 
Since \(\E[\NN(x)]\) is the optimal constant predictor for \(P_\NN\), we have \(L_{P_\NN}(h) \le L_{P_\NN}(\E[\NN(x)]) - \varepsilon \le L_{P_\NN}(\E[\phi(\gamma \langle w, x \rangle)]) - \varepsilon \). 
Similarly to the proof of Theorem~\ref{thm:CLWE-to-phi}, compute \begin{align*}
    |L_{P_\NN} (h) - L_{P_\phi}(h) | &= \left| \E_{x\sim N(0, I_d)}(h(x) - \NN(x))^2 - \E_{x\sim N(0, I_d)} (h(x) - \phi(\gamma \langle w, x \rangle))^2 \right| \\
    &= \left| \E_x \left[ \left(\NN(x) - \phi(\gamma \langle w, x \rangle)\right) \left(\NN(x) + \phi(\gamma \langle w, x \rangle) - 2h(x)\right) \right] \right| \\
    &\le \E_x \left| \NN(x) - \phi(\gamma \langle w, x \rangle) \right|.
\end{align*}
We know from the proof of Lemma~\ref{lem:phi-nn-tv-close} that \(\E_x |\NN(x) - \phi(\gamma \langle w, x \rangle)| \le d^{-\omega(1)} \le \varepsilon / 4\). 
Thus, \(|L_{P_\NN} (h) - L_{P_\phi}(h) | \le \varepsilon /4\). 
Let \(c = \E[\phi(\gamma \langle w, x \rangle)]\in [-1/4, 1/4]\). Then by the same argument, \begin{align*}
    &|L_{P_\NN}(c) - L_{P_\phi}(c)| = \left| \E_{x \sim N(0, I_d)}(c - \NN(x))^2 - \E_{x \sim N(0, I_d)}(c - \phi(x))^2 \right| \\
    &\qquad  = \left| \E_x \left[ \left(\NN(x) - \phi(\gamma \langle w, x \rangle)\right) \left(\NN(x) + \phi(\gamma \langle w, x \rangle) - 2c\right) \right] \right| \\
    &  \qquad \qquad \le \E_x \left| \NN(x) - \phi(\gamma \langle w, x \rangle) \right| \le \varepsilon /4.
\end{align*}
Therefore, with probability at least \(2/3 - 0.01\), we have \(L_{P_\phi}(h) \le L_{P_\NN}(h) + \frac{\varepsilon }{4} \le L_{P_\NN}[\E[\phi(\gamma \langle w, x \rangle)]] - \frac{3\varepsilon }{4} \le L_{P_\phi}[\E[\phi(\gamma \langle w, x \rangle)]] - \frac{\varepsilon }{2}\). 
\end{proof}

The final step is to combine this with the hardness of learning $\mathcal{F}_\gamma^{\phi}$ from Theorem~\ref{thm:CLWE-to-phi} with the equivalence of learning $\mathcal{F}_\gamma^{\phi}$ and $\mathcal{F}_R^{\NN}$ to get the following result, which directly implies Theorem~\ref{thm:mainTheorem_0}.

\begin{theorem} \label{thm:mainTheorem}
Let \(d\in \N\), \(\sigma = \poly(d)^{-1}\), \(\varepsilon = \poly(d)^{-1}\), and \(R = \omega(\sqrt{d \log d})\). Then a polynomial-time algorithm that \(\varepsilon \)-weakly learns the function class $\mathcal{F}_R^\NN,$ defined in \eqref{eq:NN_class} over Gaussian inputs \(x \overset{\iid}{\sim} N(0,I_d)\) under Gaussian noise \(\xi\overset{\iid}{\sim} N(0,\sigma^2)\) implies a polynomial-time quantum algorithm that approximates $\mathrm{GapSVP}$ to within polynomial factors. 
\end{theorem}
\begin{proof}
Let \(\gamma = 2\sqrt{d}\) and \(\phi(x) = |x-3/4- \left\lfloor x-1/4 \right\rfloor|-1/4\), which is \(1\)-Lipschitz, \(1\)-periodic, and \(\phi(x) \in [-1/4,1/4]\) for all \(x\in \R\). Then by Lemma~\ref{lem:phi-to-nn}, there is a polynomial-time algorithm that \(\frac{\varepsilon }{2}\)-weakly learns the function class \(\mathcal{F}_{\gamma}^\phi = \{f_{\gamma,w}(x) = \phi(\gamma \langle w, x \rangle) \mid w\in S^{d-1}\}\) over the same input and noise distribution. Then by Corollary~\ref{cor:phi-solves-SVP}, there is a polynomial-time quantum algorithm that approximates SVP to within polynomial factors. 
\end{proof}

\section{Super-Polynomially Small Noise} \label{sec:superpolynomial}

In this section we show that our lower bound holds even if we make the noise negligible, i.e., smaller than any inverse polynomial in $d$. Even in this very low noise regime, any algorithm for learning 1-hidden layer neural networks with Gaussian input implies a breakthrough in cryptography and algorithmic theory of lattices. We make this claim precise below.
\medskip

\noindent If we remove the restriction of \(\sigma = \poly(1/d)\) in Theorem~\ref{thm:mainTheorem}, then using the same outline of the proof we can show the following lemma that reduces learning 1-hidden layer neural networks to $\clwe$ even when $\beta$ is negligible.

\begin{lemma}
\label{lem:clwe-nn-smallnoise}
Let \(d\in \N\), \(\gamma = \poly(d)\) with \(\gamma = \omega(\sqrt{\log d})\), \(\sigma = \sigma(d) \in (0, 1)\), \(\varepsilon = \poly(d)^{-1}\), and \(R = \gamma \sqrt{\omega(\log d) + 2\log(1/\sigma)}\). Then a polynomial-time algorithm that \(\varepsilon \)-weakly learns the function class $\mathcal{F}_R^\NN,$ defined in \eqref{eq:NN_class} over Gaussian inputs \(x \overset{\iid}{\sim} N(0,I_d)\) under Gaussian noise \(\xi\overset{\iid}{\sim} N(0,\sigma^2)\) implies a polynomial-time algorithm for \(\clwe_{\beta,\gamma}\) for any \(\beta \le \sigma/ \poly(d)\). 
\end{lemma}
\begin{proof}
Let \(\phi(x) = |x-3/4- \left\lfloor x-1/4 \right\rfloor|-1/4\), which is \(1\)-Lipschitz, \(1\)-periodic, and \(\phi(x) \in [-1/4,1/4]\) for all \(x\in \R\). Then by Lemma~\ref{lem:phi-to-nn}, there is a polynomial-time algorithm that \(\frac{\varepsilon }{2}\)-weakly learns the function class \(\mathcal{F}_{\gamma}^\phi = \{f_{\gamma,w}(x) = \phi(\gamma \langle w, x \rangle) \mid w\in S^{d-1}\}\) over the same input and noise distribution. Then by Theorem~\ref{thm:CLWE-to-phi}, there is a polynomial-time algorithm for \(\clwe_{\beta,\gamma}\) for any \(\beta \le \sigma/ \poly(d)\). 
\end{proof}

Therefore an algorithm for learning 1-hidden layer neural network implies that we can solve $\clwe_{\beta, \gamma}$ as long as $\beta$ is smaller than $\sigma/\poly(d)$. The next step is to connect an algorithm for $\clwe_{\beta, \gamma}$ to an algorithm for worst-case lattice problems even when $\beta$ is negligible. For the case $\beta \ge \poly(d)^{-1}$ we used the $\clwe$ hardness fromxs Theorem~\ref{thm:clwe-hardness}x, which requires \(\gamma/\beta = \poly(d)\). Nevertheless, we can bypass this condition by using the following recent theorem that reduces classical LWE to CLWE from \cite{gupte2022continuous}, together with the stronger quantum reduction from worst-case lattice problem to LWE due to \cite{regev2005lwe}.

\begin{theorem}[{\cite[Corollary 5]{gupte2022continuous}}]
\label{thm:LWE-to-CLWE}
Let \(d, n, q \in \N\), \(\gamma, \beta,\sigma' > 0\). Then for some constant $c>0,$ a polynomial-time algorithm for \(\CLWE_{\beta,\gamma}\) in dimension \(d\) implies a polynomial-time algorithm for \(\lwe_{q, D_{\Z,\sigma'}}\) in dimension \(n\), for \begin{align*}
    \gamma &= \omega(\sqrt{d \log d}), \qquad \beta \geq c \left( \frac{\sigma' \sqrt{d}}{q} \right),
\end{align*}
as long as \(\log(q) / 2^n = o(\poly(n)^{-1})\), \(d \ge 2 n\log q + \omega(\log n)\), and \(\sigma' \ge \omega(\sqrt{\log n})\). 
\end{theorem}

\begin{theorem}[{\cite[Theorem 3.1, Lemma 3.20]{regev2005lwe}}]
\label{thm:SVP-to-LWE}
Let \(n,q\in \N\), \(\alpha \in (0, 1)\) such that \(\alpha q > 2\sqrt{n}\). Then a polynomial-time algorithm for \(\lwe_{q, D_{\Z, \alpha q}}\) in dimension \(n\) implies a polynomial-time quantum algorithm for \(O(n / \alpha)\)-GapSVP in dimension \(n\).
\end{theorem}

If we combine these two results we get the following reduction from GapSVP to $\clwe$ even for super-polynomially small $\beta$.

\begin{corollary}
\label{cor:SVP-LWE-CLWE}
Let \(d, n \in \N\), \(\gamma>0\), \(\beta \in (0, 1)\). Then a polynomial-time algorithm for \(\CLWE_{\beta,\gamma}\) in dimension \(d\) implies a polynomial-time quantum algorithm for \(O(n \sqrt{d}/\beta)\)-GapSVP in dimension \(n\), if \(\gamma = \omega(\sqrt{d \log d})\), \(\log(1/\beta) \le \poly(n)\), and \( 3n\log(d/\beta) \le d \le \poly(n)\). 
\end{corollary}
\begin{proof}
For the constant $c>0$ from Theorem~\ref{thm:LWE-to-CLWE}, let \(\alpha = c^{-1} \beta / \sqrt{d}\), \(q = 2d / \beta\), \(\sigma' = \alpha q\).

Indeed, $\gamma$ directly satisfies the same assumption, $\beta$ satisfies $\beta \geq c \sigma'\sqrt{d}/q,$ and $q$ satisfies $$\log (q)=\log(2d/\beta)=O\left(\log n+\log (1/\beta)\right) \leq \poly(n)=o(2^n (\poly(n))^{-1})$$ and $d$ satisfies $$d\ge 3n \log(d/\beta) \geq 2n\log (q)+\omega(\log n).$$ Finally, also clearly $\sigma'=c^{-1}\sqrt{d}=\omega(\sqrt{\log n})$ and $\alpha q=c^{-1}\sqrt{d}\ge c^{-1}\sqrt{3n\log(d/\beta)} >2\sqrt{n}.$

Then by Theorem~\ref{thm:LWE-to-CLWE}, there is a polynomial-time algorithm for \(\lwe_{q, D_{\Z, \alpha q}}\) in dimension \(n\). Further, since \(\alpha q >2\sqrt{n}\), by Theorem~\ref{thm:SVP-to-LWE}, there is a polynomial-time algorithm for \(O(n\sqrt{d}/\beta)\)-GapSVP algorithm. 
\end{proof}

Finally, we can use Corollary~\ref{cor:SVP-LWE-CLWE} instead of Theorem~\ref{thm:clwe-hardness} in the proof of Theorem~\ref{thm:mainTheorem} to get the following result.

\begin{theorem} \label{thm:mainTheoremSmallError}
Let \(n\in \N\),  \(\sigma \ge e^{-\poly(n)}\), \(\Omega(n\log(n/\sigma)) \le d \le \poly(n)\), \(\varepsilon = \poly(d)^{-1}\), and \(R = \omega(\sqrt{d\log d\cdot \log (d/\sigma)})\). Then a polynomial-time algorithm that \(\varepsilon \)-weakly learns the function class $\mathcal{F}_R^\NN$, defined in \eqref{eq:NN_class} over Gaussian inputs \(x \overset{\iid}{\sim} N(0,I_d)\) under Gaussian noise \(\xi\overset{\iid}{\sim} N(0,\sigma^2)\) implies a polynomial-time quantum algorithm for \((\poly(n)/\sigma)\)-GapSVP in dimension \(n\). 
\end{theorem}
\begin{proof}
Let \(\beta = \sigma/\poly(d)\), \(\gamma = \omega(\sqrt{d \log d})\), then by Lemma~\ref{lem:clwe-nn-smallnoise}, there is a polynomial-time algorithm for \(\CLWE_{\beta,\gamma}\). By Corollary~\ref{cor:SVP-LWE-CLWE}, there is a polynomial-time quantum algorithm for GapSVP with factor \(O(n\sqrt{d}/\beta) = \poly(n)/\sigma\) in dimension \(n\). 
\end{proof}

By choosing \(\sigma = 2^{-d^{\eta}}\) for constant \(\eta\in (0, 1/2)\), we can get the following corollary.

\begin{corollary}
\label{cor:sup-poly-noise-implies-sub-exp-approx}
For constant \(\eta \in (0, 1/2)\), let \(\sigma = 2^{-d^{\eta}}\), \(\varepsilon = \poly(d)^{-1}\), and \(R = \omega(\sqrt{d^{1+\eta}\log d})\). Then for \(n \in \N\) with \(n= \Theta(d^{1-\eta})\), a polynomial-time algorithm that \(\varepsilon \)-weakly learns the function class $\mathcal{F}_R^\NN$, defined in \eqref{eq:NN_class} over Gaussian inputs \(x \overset{\iid}{\sim} N(0,I_d)\) under Gaussian noise \(\xi\overset{\iid}{\sim} N(0,\sigma^2)\) implies a polynomial-time quantum algorithm for \(2^{O(n^{\frac{\eta}{1-\eta}})}\)-GapSVP in dimension \(n\). 
\end{corollary}
\begin{proof}
Since \(\eta\) is a constant, indeed we have \(\sigma \ge e^{-\poly(n)}\), \[d = \Theta(n^{\frac{1}{1-\eta}}) \ge \Omega(n \cdot n^{\frac{\eta}{1-\eta}}) \ge \Omega(n d^{\eta}) = \Omega(n \log (n/\sigma)),\] and \(R = \omega(\sqrt{d\log d\cdot \log(d/\sigma)})\). Thus, from Theorem~\ref{thm:mainTheoremSmallError}, there is a polynomial-time quantum algorithm for GapSVP with factor \(\poly(n)/\sigma = 2^{O(n^{\frac{\eta}{1-\eta}})}\) in dimension \(n\). 
\end{proof}
According to Corollary~\ref{cor:sup-poly-noise-implies-sub-exp-approx}, Theorem~\ref{thm:mainTheoremSmallError} shows an interesting hardness result even when $\sigma$ is super-polynomially small. To make the connection more precise, observe that for any $\delta \in (0,1)$, if we set \(\eta := \frac{\delta}{1+\delta}\in (0, 1/2)\), and \(\sigma = 2^{-d^{\eta}}\), then, due to Corollary~\ref{cor:sup-poly-noise-implies-sub-exp-approx}, the existence of a polynomial-time algorithm that weakly learns one hidden layer neural networks with polynomial width under Gaussian noise with only \(2^{-d^ {\eta} }\) standard deviation, implies a polynomial-time quantum algorithm for \(2^{O(n^{\delta })}\)-GapSVP in dimension \(n\) --- a problem which is considered hard in cryptography and algorithmic theory of lattices. We remind the reader, that the main reason behind this conjectured hardness is that the state-of-the-art (since 1982) powerful LLL algorithm for $\mathrm{GapSVP}$ is only able to achieve an $2^{\Theta(n)}$-approximation, and any improvement on it would be considered a major breakthrough. We also highlight that any such algorithm would break in fact several breakthrough cryptographic constructions such as the recent celebrated result of \cite{jain2021indistinguishability}.

\section{Conclusions}
In this paper, we proved the hardness of learning one hidden layer neural networks with width \(\omega(\sqrt{d \log d})\) under Gaussian input and any inverse-polynomially small Gaussian noise, assuming the hardness of GapSVP with polynomial factors. En route, we proved the hardness of learning Lipschitz periodic functions under Gaussian input and any inverse-polynomially small Gaussian noise. This improves a similar result from \cite{SZB21-cosine-learning}, which proved the hardness for inverse-polynomially small \textit{adversarial} noise. 

Moreover, if we assume the hardness of \(2^{O(d^{\delta })}\)-GapSVP for \(\delta \in (0, 1)\), we also get the hardness of learning one hidden layer neural networks with polynomial width under Gaussian noise with \(2^{-d^{\eta}}\) variance, where \(\eta = \frac{\delta}{1+\delta}\in (0, 1/2)\). 

\printbibliography

\end{document}